\documentclass{IOS-Book-Article}

\usepackage{mathptmx}
\usepackage{soul}\setuldepth{article}

\usepackage[noend]{algpseudocode}
\usepackage{algorithm}
\usepackage{algorithmicx}
\usepackage{microtype}
\algrenewcommand\textproc{}

\usepackage{paralist}
\usepackage{enumitem}
\usepackage{amssymb}
\usepackage{amsfonts}
\usepackage{mathtools}
\usepackage{autoaligne}
\usepackage{multirow}
\usepackage{xspace}
\usepackage{pgf, tikz}
\usetikzlibrary{arrows, automata}

\newcommand{\ABAplus}{ABA\textsuperscript{+} }
\newcommand{\stcomp}[1]{\overline{#1}}

\newcommand{\assump}{\mathcal{A}}
\newcommand{\lang}{\mathcal{L}}
\newcommand{\contrary}{\mathcal{^-}}

\newcommand{\rules}{\mathcal{R}}

\newtheorem{theorem}{Theorem}[section]

\newtheorem{example}{Example}[section]
\newtheorem{definition}{Definition}[section]

\newtheorem{proof}{Proof}[section]
\def\hb{\hbox to 10.7 cm{}}

\begin{document}

\pagestyle{headings}
\def\thepage{}

\begin{frontmatter}              

\title{Preference Elicitation in Assumption-Based Argumentation \thanks{This  work  was supported by EPSRC grant EP/P011829/1, Supporting Security Policy with Effective Digital Intervention (SSPEDI)}}


\author[A]{\fnms{Quratul-ain} \snm{Mahesar}},
\author[B]{\fnms{Nir} \snm{Oren}}
and
\author[B]{\fnms{Wamberto W.} \snm{Vasconcelos}}

\address[A]{Department of Informatics, King's College London,
 UK}
\address[B]{Department of Computing Science, University of Aberdeen, Aberdeen, UK}

\begin{abstract}
Various structured argumentation frameworks utilize preferences as part of their standard inference procedure to enable reasoning with preferences. In this paper, we consider an inverse of the standard reasoning problem, seeking to identify what preferences over assumptions could lead to a given set of conclusions being drawn. We ground our work in the Assumption-Based Argumentation (ABA) framework, and present an algorithm which computes and enumerates all possible sets of preferences over the assumptions in the system from which a desired conflict free set of conclusions can be obtained under a given semantic. After describing our algorithm, we establish its soundness, completeness and complexity.
\end{abstract}

\end{frontmatter}

\section{Introduction}
Assumption-based argumentation (ABA)~\cite{aba_Dung2009,aba_Toni14} is a structured argumentation formalism where knowledge is represented through a deductive system comprising of a formal language and inference rules. Defeasible information is represented in the form of special entities in the language  called \emph{assumptions}. Attacks are then constructively defined over sets of assumptions whenever one set of assumptions supports the deduction of the contrary of some assumption in a different set of assumptions. ABA is equipped with different semantics for determining `winning' sets of assumptions. In turn, sets of `winning' (aka acceptable) arguments can be determined from the `winning' sets of assumptions, since the assumption-level and argument-level views are fully equivalent in ABA.

\ABAplus~\cite{CyrasT16} extends ABA with preferences to assist with discrimination among conflicting alternatives. Assumptions are the only defeasible component in ABA and therefore preferences are defined over assumptions rather than arguments as is the case in preference-based argumentation frameworks (PAFs)~\cite{AMGOUD2014585}. Unlike PAFs --- that use attack reversal with preferences given over arguments --- \ABAplus uses attack reversal with preferences given over assumptions (at the object level), that form the support of arguments. Preferences are integrated directly into the attack relation, without lifting preferences from the object level to either the argument or the extension levels (as is done in systems such as ASPIC\textsuperscript{+}~\cite{ModgilP14}).

ABA semantics allow one to compute an acceptable or winning set of arguments, and the semantics of \ABAplus  extend ABA semantics to take  preferences over assumptions into account. In this paper we focus on the problem of eliciting the preferences which unify an ABA framework with its \ABAplus counterpart, seeking to answer the following questions.
\setlist{nolistsep}
\begin{enumerate}[noitemsep]
\item What are the possible preferences over assumptions for a given conflict-free extension in an ABA framework that ensure the acceptability of the extension using \ABAplus semantics?
\item How do we find a combination of all possible sets of such preferences?
\item What are the unique and common preferences for a given extension compared to all other extensions (under a multi-extension semantics)?
\end{enumerate}

The following example suggests one application for these ideas.
\begin{example}
\label{example:journey}
(Journey Recommendation) Alice uses an online journey recommendation system to find  options for travelling between two cities. The system suggests two options: by ground or air. If Alice chooses ground, the system can infer that Alice likes to travel by train and not by plane. The system can learn Alice's preferences from this interaction and suggest the right recommendation next time. In the future, the system can justify its recommendation to Alice by arguing that travelling by ground is the best option for her since she prefers to travel by train. Furthermore, if the train company or airline states that their trains or planes are ``air-conditioned vehicles", then the system could learn an additional implicit preference, namely that Alice prefers to travel on an ``air-conditioned vehicle".
\end{example}
Preferences not only help in finding the best solutions (via winning arguments) suitable for a particular user but also in justifying such solutions. In turn, computing all possible sets of preferences helps in finding all justifications.
Previous work on the topic~\cite{mahesar_prima18} --- which we build on --- considered only abstract argumentation frameworks (namely Dung's AAFs and PAFs), and this paper extends the state of the art by considering a structured setting, building on ABA and \ABAplus frameworks. Moving to a structured setting is non-trivial and requires additional issues to be addressed that do not arise in the abstract cases.

 The remainder of this paper is structured as follows. In Section~\ref{sec:background}, we present the background on ABA and \ABAplus frameworks. In Section~\ref{sec:approach}, we present our proposed approach and algorithm for computing all possible sets of preferences for a given extension and ABA framework, and we prove the soundness and completeness of the algorithm. In Section~\ref{sec:evaluation}, we present the evaluation and results. In Section~\ref{sec:related-work}, we present the related work. Finally, we conclude and suggest future work in Section~\ref{sec:conclusion}.

\section{Background}
\label{sec:background}
We begin by introducing ABA and \ABAplus, with the following material taken from~\cite{aba_Dung2009,aba_Toni14,CyrasT16}. An ABA framework is a tuple $(\mathcal{L}, \mathcal{R}, \mathcal{A}, ^-)$, where:
\setlist{nolistsep}
\begin{enumerate}[noitemsep]
	\item $(\mathcal{L}, \mathcal{R})$ is a deductive system with $\mathcal{L}$ a language (a set of sentences) and $\mathcal{R}$ a set of rules of the form $\varphi_0 \gets \varphi_1,...,\varphi_m$ with $m \geq 0$ and $\varphi_i \in \mathcal{L}$ for $i \in \{0,...,m\}$. 
		$\varphi_0$ is referred to as the \textit{head} of the rule, and
		$\varphi_1,...,\varphi_m$ is referred to as the \textit{body} of the rule. 
		If $m=0$, then the rule $\varphi_0 \gets \varphi_1,...,\varphi_m$ is said to have an \textit{empty body}, and is written as $\varphi_0 \gets \top$, where $\top \notin \mathcal{L}$.
	\item $\mathcal{A} \subseteq \mathcal{L}$ is a non-empty set, whose elements are referred to as \textit{assumptions};
	\item $^-: \mathcal{A} \rightarrow \mathcal{L}$ is a total map where for $\alpha \in \mathcal{A}$, the $\mathcal{L}$-sentence $\bar{\alpha}$ is referred to as the \textit{contrary} of $\alpha$.
\end{enumerate}

A \textit{deduction}\footnote{We only consider \textit{flat} ABA frameworks, where assumptions cannot be deduced from other assumptions.} for $\varphi \in \mathcal{L}$ supported by $S \subseteq \mathcal{L}$ and $R \subseteq \mathcal{R}$, denoted by $S \vdash_R \varphi$, is a finite tree with:
\begin{inparaenum}[(i)]
	\item the root labelled by $\varphi$,
	\item leaves labelled by $\top$ or elements from $S$,
	\item the children of non-leaf nodes $\psi$ labelled by the elements of the body of some rule from $\mathcal{R}$ with head $\psi$, and $\mathcal{R}$ being the set of all such rules.
\end{inparaenum}
For $E \subseteq \mathcal{L}$, the \textit{conclusions} $Cn(E)$ of $E$ is the set of sentences for which deductions supported by subsets of $E$ exist, i.e., $Cn(E) = \{\phi \in \mathcal{L}: \exists S \vdash_R \phi, S \subseteq E, R \subseteq \mathcal{R} \}$.

 The semantics of ABA frameworks are defined in terms of sets of assumptions meeting desirable requirements, and are given in terms of a notion of \textit{attack} between sets of assumptions. A set $A \subseteq \mathcal{A}$ \textit{attacks} a set $B \subseteq \mathcal{A}$ of assumptions, denoted by $A \rightsquigarrow B$ iff there is a deduction $A' \vdash_R \stcomp{\beta}$, for some $\beta \in B$, supported by some $A' \subseteq A$ and $R \subseteq \mathcal{R}$. If it is not the case that $A$ attacks $B$, then we denote this by $A \not\rightsquigarrow B$.

For $E \subseteq \mathcal{A}$
	$E$ is \textit{conflict-free} iff $E \not\rightsquigarrow E$;
	$E$ \emph{defends} $A \subseteq \mathcal{A}$ iff for all $B \subseteq \mathcal{A}$ with $B \rightsquigarrow A$ it holds that $E \rightsquigarrow B$.
A set $E \subseteq \mathcal{A}$ of assumptions (also called an \textit{extension}) is
	\textit{admissible} iff $E$ is conflict-free and defends itself;
	\textit{preferred} iff $E$ is $\subseteq$-maximally admissible;
	\textit{stable} iff $E$ is conflict-free and $E \rightsquigarrow \{\beta\}$ for every $\beta \in \mathcal{A} \backslash E$;	
	\textit{complete} iff $E$ is admissible and contains every set of assumptions it defends;
	\textit{grounded} iff $E$ is a $\subseteq$-minimal complete extension.

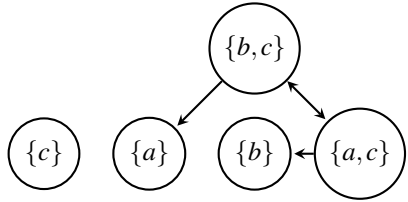
\begin{figure}
	\centering
	\begin{tikzpicture}[
	> = stealth, 
	shorten > = 1pt, 
	auto,
	node distance = 1.4cm, 
	thick 
	]
	
	\tikzstyle{every state}=[
	draw = black,
	thick,
	fill = white,
	minimum size = 4mm
	]
	\node[state] (C) {$\{c\}$};
	\node[state] (A) [right of=C] {$\{a\}$};
	\node[state] (B) [right of=A] {$\{b\}$};
	\node[state] (X) [above of=B] {$\{b,c\}$};
	\node[state] (Y) [right of=B] {$\{a,c\}$};
	
	\path[->] (X) edge node {} (A);
	\path[->] (Y) edge node {} (B);
	\path[<->] (X) edge node {} (Y);        
	
	\end{tikzpicture}
	\caption{Example assumption graph $\mathcal{G}_1$}
	\label{fig:aba_arg_graph}
\end{figure}

\begin{example}
\label{example:aba}
	Example~\ref{example:journey} can be represented in ABA as follows. Let there be an ABA framework $\mathcal{F} = (\lang, \rules, \assump, \contrary)$ with
		language $\lang = \{ a, b, c, d, e, f \}$; 
		set of rules $\rules = \{ d \leftarrow a,c; \: e \leftarrow b,c \}$; 
		set of assumptions $\assump = \{ a, b, c \}$. 
		Contraries are then given by: $\stcomp{a}=e, \stcomp{b}=d, \stcomp{c}=f$.
	Here $a$ and $b$ stand for ``plane" and ``train" respectively, and $c$ stands for ``air-conditioned vehicle". The contraries of $a$ and $b$, namely $e$ and $d$ stand for ``ground" and ``air", and the contrary of $c$, $f$ stands for ``hot".
	
	$\mathcal{F}$ is illustrated in Figure~\ref{fig:aba_arg_graph}. Here, nodes represent sets of assumptions and directed edges denote attacks. $\{b,c\}$ attacks $\{a\}$, $\{a, c\}$ attacks $\{b\}$, $\{a, c\}$ and $\{b, c\}$ attack each other and $\{c\}$ is un-attacked and does not attack anyone.

	$\mathcal{F}$ has two preferred and stable extensions $\{a,c\}$ and $\{b,c\}$, with conclusions $\mathit{Cn}(\{a,c\}) = \{a, c, d\}$ and $\mathit{Cn}(\{b,c\}) = \{b, c, e\}$, respectively. $\mathcal{F}$ has a unique grounded extension $\{c\}$, with conclusions $\mathit{Cn}(\{c\}) = \{c\}$. Furthermore, all of $\{c\}$, $\{a,c\}$ and $\{b,c\}$ are complete extensions.

\end{example}

\ABAplus is an extension of the ABA framework with a preference ordering $\leq$ on the set $\mathcal{A}$ of assumptions, defined as follows.
\begin{definition}
	An \ABAplus framework is a tuple $(\mathcal{L}, \mathcal{R}, \mathcal{A}, ^-, \leq)$, where $(\mathcal{L}, \mathcal{R}, \mathcal{A}, ^-)$ is an ABA framework and $\leq$ is a pre-ordering defined on $\mathcal{A}$.
\end{definition}

$<$ is the strict counterpart of $\leq$, and is defined as $\alpha < \beta$ iff $\alpha \leq \beta$ and $\beta \nleq \alpha$, for any $\alpha$ and $\beta$. 
The attack relation in \ABAplus is then defined as follows. 
\begin{definition}
	Let $A, B \subseteq \mathcal{A}$,  $A$ $<$-attacks $B$, denoted by $A \rightsquigarrow_< B $ iff either:
	\begin{inparaenum}[(i)]
		\item  there is a deduction $A' \vdash_R \bar{b}$, for some $b \in B$, supported by $A' \subseteq A$, and $\nexists a' \in A'$ with $a' < b$; or
		\item there is a deduction $B' \vdash_R \bar{a}$, for some $a \in A$, supported by $B' \subseteq B$, and $\exists b' \in B'$ with $b' < a$.
	\end{inparaenum}
\end{definition}

The $<$-attack formed in the first  point above is called a \textit{normal attack}, while the $<$-attack formed  in by the second  point  is a \textit{reverse attack}. We write $A \not\rightsquigarrow B$ to denote that $A$ does not $<$-attack $B$. The \textit{reverse attack} ensures that the original conflict is preserved and that conflict-freeness of the extensions is attained, which is the basic condition imposed on acceptability semantics. Flatness is defined as for standard ABA frameworks, and we assume that the ABA and \ABAplus frameworks we deal with are flat.

Furthermore, for $A \subseteq \mathcal{A}$, $A$ is $<$-\textit{conflict-free} iff $A \not\rightsquigarrow_< A$; also, $A$ $<$-\textit{defends} $A' \subseteq \mathcal{A}$ iff for all $B \subseteq \mathcal{A}$ with  $B \rightsquigarrow_< A'$ it holds that  $A \rightsquigarrow_< B$. In \ABAplus, an extension $E \subseteq \mathcal{A}$ is 
	$<$-\textit{admissible} iff $E$ is $<$-\textit{conflict-free} and $<$-\textit{defends} itself;
	$<$-\textit{preferred} iff $E$ is $\subseteq$-maximally $<$-admissible.
	$<$-\textit{stable} iff $E$ is $<$-conflict-free and $E \rightsquigarrow_< \{\beta\}$ for every $\beta \in \mathcal{A} \backslash E$;
	$<$-\textit{complete} iff $E$ is $<$-\textit{admissible} and contains every set of assumptions it $<$-\textit{defends};	
	$<$-\textit{grounded} iff $E$ is a $\subseteq$-minimal $<$-complete set of assumptions.
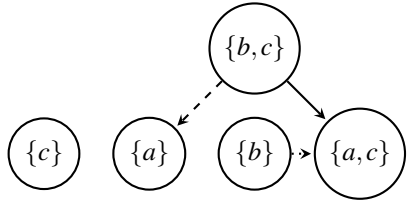
\begin{figure}
	\centering
	\begin{tikzpicture}[
	> = stealth, 
	shorten > = 1pt, 
	auto,
	node distance = 1.4cm, 
	thick 
	]
	
	\tikzstyle{every state}=[
	draw = black,
	thick,
	fill = white,
	minimum size = 4mm
	]
	\node[state] (C) {$\{c\}$};
	\node[state] (A) [right of=C] {$\{a\}$};
	\node[state] (B) [right of=A] {$\{b\}$};
	\node[state] (X) [above of=B] {$\{b,c\}$};
	\node[state] (Y) [right of=B] {$\{a,c\}$};
	
	\path[dashed,->] (X) edge node {} (A);
	\path[dotted,->] (B) edge node {} (Y);
	\path[->] (X) edge node {} (Y);        
	
	\end{tikzpicture}
	\caption{Example assumption graph $\mathcal{G}_2$}
	\label{fig:aba_arg_graph2}
\end{figure}
\begin{example}
	We extend the ABA framework of Example~\ref{example:aba} to \ABAplus by including preferences over assumptions. Let there be an \ABAplus framework $\mathcal{F^+} = (\lang, \rules, \assump, \contrary, \leq)$ with:
		$\lang = \{ a, b, c, d, e, f \}$;
		$\rules = \{ d \leftarrow a,c; \: e \leftarrow b,c \}$;
		$\assump = \{ a, b, c \}$;
		$\stcomp{a}=e, \stcomp{b}=d, \stcomp{c}=f$.
		$a < b$ (i.e., $<$ is a pre-order with $a \leq b, b \not\leq a$). 
Here $a<b$ represents the preference that Alice prefers to travel by train compared to plane. $\mathcal{F^+}$ is graphically represented in Figure~\ref{fig:aba_arg_graph2} where the nodes contain sets of assumptions and directed edges denote attacks. Dashed arrows indicate normal attacks, dotted arrows indicate reverse attacks and solid arrows indicate $<$-attacks that are both normal and reverse.

In $\mathcal{F^+}$, $\{b,c\}$ supports an argument for the contrary $e$ of $a$, and no assumption in $\{b,c\}$ is strictly less preferred than $a$. Thus, $\{b,c\}$ $<$-attacks $\{a\}$, as well as any set containing $a$ via normal attack. On the other hand, $\{a,c\}$ (supporting an argument for the contrary $d$ of $b$) is prevented from $<$-attacking $\{b\}$, due to the preference $a<b$. Instead $\{b\}$, as well as any set containing $b$, $<$-attacks $\{a,c\}$ via reverse attack. Overall, $\mathcal{F^+}$ has a unique $<$-$\sigma$ extension (where $\sigma \in \{\mathit{grounded, preferred, stable, complete}\})$, namely $E= \{b,c\}$, with conclusions $\mathit{Cn}(E) = \{b,c,e\}$.

\end{example}

\section{Approach and Algorithms for Preference Elicitation}
\label{sec:approach}
In this section, we present an extension-based approach for preference elicitation within an ABA framework. The core idea is that by observing what conclusions a reasoner draws based on their knowledge base, we can reason about what preferences over assumptions they  hold. Thus, given an ABA framework, a semantics $\sigma$, and a set of assumptions that are considered justified, we identify the preferences necessary to obtain these assumptions within an \ABAplus $<$-$\sigma$ extension. 

For any two assumptions $a$ and $b$ in an ABA framework, we use the strict preference relation $a<b$ to denote that $a$ is strictly less preferred to $b$, i.e., $a$ is of lesser strength than $b$, and we use the preference relation $a=b$ to denote that $a$ and $b$ are of equal strength or preference. There are then three possible cases for which the preferences between assumptions are computed for a given conflict-free extension $E$ in an ABA framework.
\smallskip

\noindent\textbf{Case 1:} Suppose $a \in \mathcal{A}$, $\beta \subseteq \mathcal{A}$ and $a \in E$, $\nexists b \in \beta \: s.t. \: b \in E$ such that $a$ is attacked by $\beta$, i.e., $\beta \vdash_R \bar{a}$, where $R \subseteq \mathcal{R}$ and $a$ is not defended by any unattacked assumptions other than $a$ in the extension. 
Then $\forall b \in \beta$, it is the case that $b \leq a$ (i.e., $b < a$ or $b = a$, and $a \not < b$) and $\exists b' \in \beta$ such that $b' < a$.

For example, let $\beta = \{b_1, b_2\}$ and $\beta \vdash_R \bar{a}$ be given, then we have the following set of sets of all possible preferences: $\{ \{b_1<a, b_2=a \}, \{b_1=a, b_2<a \}, \{b_1<a, b_2<a \} \}$. This ensures that there is at least one assumption in $\beta$ that is strictly less preferred than the assumption $a$ whose contrary $\beta$ supports deductively, and $a$ is not strictly less preferred than any assumption in $\beta$. 
\smallskip

\noindent \textbf{Case 2:} Suppose $\alpha \subseteq \mathcal{A}$, $b \in \mathcal{A}$ and $\exists a \in \alpha \; s.t. \; a \in E$, $b \notin E$, and suppose $\alpha$ attack $b$, i.e., $\alpha \vdash_R \bar{b}$, where $R \subseteq \mathcal{R}$ and $b$ does not attack $\alpha$. 
Then $\forall a \in \alpha$, it is the case that $b \leq a$ (i.e., $b < a$ or $b = a$, and $a \not < b $)

For example, let $\alpha = \{a_1, a_2\}$ and $\alpha \vdash_R \bar{b}$ be given, then we have the following set of sets of all possible preferences: $\{ \{b<a_1, b=a_2 \}, \{b=a_1, b<a_2 \}, \{b<a_1, b<a_2 \}, \{b=a_1, b=a_2 \} \}$. This ensures that there is not a single assumption in $\alpha$ that is strictly less preferred then the assumption $b$ whose contrary $\alpha$ deductively supports. 
\smallskip

\noindent\textbf{Case 3:} Suppose $a \in \mathcal{A}$, $\beta, \gamma \subseteq \mathcal{A}$ and $a \in E$, $\exists c \in \gamma \; s.t. \; c \in E$, $\nexists b \in \beta \: s.t. \: b \in E$ where $a$, $b$ and $c$ are different assumptions, such that, $a$ is attacked by $\beta$, i.e., $\beta \vdash_R \bar{a}$ but defended by any unattacked set of assumptions $\gamma$, i.e., $\gamma \vdash_R \bar{b}$, where $b \in \beta$. 
Then $\forall b \in \beta$, it is the case that $b \leq a$ and $a \leq b$ (i.e., $(b < a)$ or $(b = a)$ or $(a < b)$). 

For example, let  $\beta = \{b_1, b_2\}$ and $\beta \vdash_R \bar{a}$ be given, and assume $\gamma \vdash_R \bar{b}$ for some $b \in \beta$. Then we have three possible set of sets of preferences: 
\begin{inparaenum}[(i)]
\item $\{ \{b_1<a, b_2=a \}, \{b_1=a, b_2<a \}, \{b_1<a, b_2<a \}  \}$, i.e., there is at least one assumption in $\beta$ that is strictly less preferred then the assumption $a$ whose contrary $\beta$ supports deductively. 
\item $\{ \{b_1=a, b_2=a \}\}$, i.e., all assumptions in $\beta$ are equal to the assumption $a$ whose contrary $\beta$ supports deductively.
\item $\{ \{a<b_1, a=b_2 \}, \{a=b_1, a<b_2 \}, \{a<b_1, a<b_2 \} \}$, i.e., assumption $a$ whose contrary $\beta$ supports deductively is strictly less preferred to at least one assumption in $\beta$. 
\end{inparaenum}
The final set of sets of preferences is the union of the above three set of sets of preferences given by $\{ \{b_1<a, b_2=a \}, \{b_1=a, b_2<a \}, \{b_1<a, b_2<a \}, \{b_1=a, b_2=a \}, \{a<b_1, a=b_2 \}, \{a=b_1, a<b_2 \}, \{a<b_1, a<b_2 \} \}$ 

Algorithm~\ref{alg:computing-preferences} exhaustively computes all possible sets of preferences over assumptions for a given input extension (a conflict-free set of assumptions) in an ABA framework using the above three cases. The input of Algorithm $1$ consists of:
\begin{inparaenum}[(i)]
	\item ABA framework.
	\item Extension $E$ which is a conflict-free set of assumptions such that $E \subseteq \mathcal{A}$, where $\mathcal{A}$ is the set of assumptions in an ABA framework. 
\end{inparaenum}
The algorithm computes and outputs a set consisting of finite sets of possible preferences between assumptions, where each set of preferences is represented as $\mathit{Prefs} = \lbrace a < b, b = c, .... \rbrace$ such that $\lbrace a,b,c \rbrace \subseteq \mathcal{A}$. 

\begin{algorithm}[!h]
	\begin{algorithmic}[1]
		\Require $\mathit{ABA}$, an ABA framework
		\Require $E$, an extension (conflict-free set of assumptions)
		\Ensure $\mathit{PSet}$, the set of sets of all possible preferences
		\Function{ComputeAllPreferences}{$\mathit{ABA},E$}
		\State $\mathit{PSet}\gets\textsc{ComputePreferences}_1(\mathit{ABA},E)$
		\State $\mathit{PSet}\gets\textsc{ComputePreferences}_2(\mathit{ABA},E,\mathit{PSet})$
		\State $\mathit{PSet}\gets\textsc{ComputePreferences}_3(\mathit{ABA},E,\mathit{PSet})$
		\State \Return $\mathit{PSet}$
		\EndFunction
	\end{algorithmic}
	\caption{Compute all preferences} 
	\label{alg:computing-preferences}
\end{algorithm}
Algorithms ~\ref{alg:computing-preferences-case1},~\ref{alg:computing-preferences-case2},~\ref{alg:computing-preferences-case3} compute case $1$, case $2$ and case $3$ preferences respectively. The following are the main steps in Algorithm~\ref{alg:computing-preferences-case1}:
\begin{inparaenum}[(i)]
	\item Line $3$: Iteratively pick a single assumption $a$ from the extension $E$.
	\item Line $4$: Find all attacking assumptions $B$ that support contrary of assumption $a$.
	\item Lines $5-24$: For each $B$, if there is no unattacked assumption $C \neq a$ that attacks $B$, then compute all possibilities where at least a single $b \in B$ is less preferred to $a$.
\end{inparaenum}

The following are the main steps in Algorithm~\ref{alg:computing-preferences-case2}:
\begin{inparaenum}[(i)]
	\item Line $3$: Iteratively pick a single assumption $a$ from the extension $E$.
	\item Line $4$: Find assumption $b$ whose contrary $A$ supports, where $a \in A$.
	\item Lines $5-10$: For all assumptions $b$ attacked by $a$, compute all possibilities where $b$ is either less preferred to $a$ or equal to $a$. Combine these preferences with case $1$ sets of preferences by adding each such preference to each preference set of case $1$.
\end{inparaenum}

The following are the main steps in Algorithm~\ref{alg:computing-preferences-case3}:
\begin{inparaenum}[(i)]
	\item Line $3$: Iteratively pick a single assumption $a$ from the extension $E$.
	\item Line $4$: Find all attacking assumptions $B$ that support contrary of the assumption $a$.
	\item Lines $5-14$: For each $B$, if there is an unattacked assumption $C \neq a$ that attacks $B$, then compute all possibilities (different sets) where:
	\begin{inparaenum}[(I)]
	\item At least a single $b \in B$ is less preferred to $a$ and ensure that $a$ cannot be less preferred to $b$ in these sets of preferences
	\item $a$ is less preferred to at least a single $b \in B$ and ensure that $b$ cannot be less preferred to $a$ in these sets of preferences
	\item For all $b \in B$, $b$ is equal to $a$.
	\end{inparaenum}
\end{inparaenum}
\begin{algorithm}[!h]
	\begin{algorithmic}[1]
		\Require $\mathit{ABA}$, an ABA framework
		\Require $E$, an extension (conflict-free set of assumptions)
		\Ensure $\mathit{PSet''}$, a set of sets of preferences		
		\Function{ComputePreferences$_1$}{$\mathit{ABA},E$}
		\State $\mathit{PSet}\gets \{ \{ \} \}$, $\mathit{PSet'}\gets\emptyset$,		
		$\mathit{PSet''}\gets\emptyset$, $\mathit{PSet'''}\gets\emptyset$
		\ForAll{$a \in E$}  
		\State $\mathit{Attackers} \gets \{B\;|\; \exists y \in Y : B \in y, B \notin E,  \; Y \vdash_R \bar{a}, \; R \subseteq \mathcal{R}\}$ 
		\ForAll{$B\in\mathit{Attackers}$}
		\State $\mathit{Defenders} \gets \{C\;|\; \exists x \in X: C \in x, C \neq a, C \in E, \; X \vdash_R \bar{b}, \; b \in B,$
		
		$\nexists z \in Z: D \in z, \; D \notin E \; s.t. \; Z \vdash_R \bar{C}, \; R \subseteq \mathcal{R}\}$ 
		
		\If{$\mathit{Defenders}=\emptyset$} 	
		\ForAll{$b \in B$}
		
		\If{$(a \neq b) \; \textit{and} \; (\nexists \mathit{Prefs''} \in \mathit{PSet''} \; s.t. \; (b<a) \in \mathit{Prefs''} \; \mathit{or} \;$ 
		
		$(b=a) \in \mathit{Prefs''} \; \mathit{or} \; (a<b) \in \mathit{Prefs''})$}

		\ForAll{$\mathit{Prefs}\in\mathit{PSet}$} 
		\State $\mathit{PSet'} \gets \mathit{PSet'} \cup \{\mathit{Prefs} \cup \{b < a\}\}$ 	
		\State $\mathit{PSet'} \gets \mathit{PSet'} \cup \{\mathit{Prefs} \cup \{b = a\}\}$ 	
		\EndFor

		\State $\mathit{PSet} \gets \mathit{PSet'}$
		\State $\mathit{PSet'} \gets \emptyset$
		
		\EndIf
		
		\EndFor
		
		\If{$\exists \mathit{Prefs} \in \mathit{PSet} \; s.t. \; \mathit{Prefs}$ has all equal relations}
		\State $\mathit{PSet} \gets \mathit{PSet}\setminus \{\mathit{Prefs}\}$
		\EndIf
		
		\If{$\mathit{PSet''} = \emptyset$}
		\State $\mathit{PSet'''} \gets \mathit{PSet}$
		\Else
		\ForAll{$\mathit{Prefs} \in \mathit{PSet}$}
		\ForAll{$\mathit{Prefs''} \in \mathit{PSet''}$}
		\State $\mathit{PSet'''} \gets \mathit{PSet'''} \cup \{ \mathit{Prefs} \cup \mathit{Prefs''}\}$
		\EndFor
		\EndFor
		\EndIf
		
		\State $\mathit{PSet''} \gets \mathit{PSet'''}$
		\State $\mathit{PSet'''}\gets\emptyset$, $\mathit{PSet}\gets \{ \{ \} \}$
		\EndIf
		\EndFor
		\EndFor
		
		\State \Return $\mathit{PSet''}$
		\EndFunction
	\end{algorithmic}
	\caption{Compute preferences (Case 1)}
	\label{alg:computing-preferences-case1}	
\end{algorithm}
\begin{algorithm}[!h]
	\begin{algorithmic}[1]
		\Require $\mathit{ABA}$, an ABA framework
		\Require $E$, an extension (conflict-free set of assumptions)
		\Require $\mathit{PSet}$, a set of sets of preferences
		\Ensure $\mathit{PSet}$, an updated set of sets of preferences		
		\Function{ComputePreferences$_2$}{$\mathit{ABA}, E, \mathit{PSet}$}
		\State $\mathit{PSet'}\gets\emptyset$		
		\ForAll{$a \in E$}  
		\State $\mathit{Attacked} \gets \{ b \; | \; b \in \mathcal{A}, \; a \in A \; s.t. \; A \vdash_R \bar{b}, \; R \subseteq \mathcal{R} \}$ 
		\ForAll{$b \in \mathit{Attacked}$} 
		\If{$(a \neq b) \; \textit{and} \; (\nexists \mathit{Prefs} \in \mathit{PSet} \; s.t. \; (b<a) \in \mathit{Prefs} \; \mathit{or} \;$ 
		
		$(b=a) \in \mathit{Prefs} \; \mathit{or} \; (a<b) \in \mathit{Prefs})$} 
		\ForAll{$\mathit{Prefs}\in\mathit{PSet}$} 
		\State $\mathit{PSet'} \gets \mathit{PSet'} \cup \{\mathit{Prefs} \cup \{b < a\}\}$ 
		\State $\mathit{PSet'} \gets \mathit{PSet'} \cup \{\mathit{Prefs} \cup \{b = a\}\}$ 
		\EndFor
		\State $\mathit{PSet} \gets \mathit{PSet'}$, $\mathit{PSet'}\gets\emptyset$
		\EndIf
		\EndFor
		\EndFor
		\State \Return $\mathit{PSet}$
		\EndFunction
	\end{algorithmic}
	\caption{Compute preferences (Case 2)}
	\label{alg:computing-preferences-case2}	
\end{algorithm}
\begin{algorithm}[!h]
	\begin{algorithmic}[1]
		\Require $\mathit{ABA}$, an ABA framework
		\Require $E$, an extension (conflict-free set of assumptions)
		\Require $\mathit{PSet}$, a set of sets of preferences
		\Ensure $\mathit{PSet}$, an updated set of sets of preferences			
		\Function{ComputePreferences$_3$}{$\mathit{ABA},E,\mathit{PSet}$}
		\State $\mathit{PSet'}\gets\emptyset$	
		\ForAll{$a \in E$}  
		\State $\mathit{Attackers} \gets \{B\;|\; \exists y \in Y : B \in y, B \notin E,  \; Y \vdash_R \bar{a}, \; R \subseteq \mathcal{R}\}$ 
		\ForAll{$B\in\mathit{Attackers}$}
		\State $\mathit{Defenders} \gets \{C\;|\; \exists x \in X: C \in x, C \neq a, C \in E, \; X \vdash_R \bar{b}, \; b \in B,$
		
		$\nexists z \in Z: D \in z, \; D \notin E \; s.t. \; Z \vdash_R \bar{C}, \; R \subseteq \mathcal{R}\}$ 
		\If{$\mathit{Defenders} \neq \emptyset$} 
		\ForAll{$b \in B$}
		\If{$(a \neq b) \; \textit{and} \; (\nexists \mathit{Prefs} \in \mathit{PSet} \; s.t. \; (b<a) \in \mathit{Prefs} \; \mathit{or} \;$ 
		
		$(b=a) \in \mathit{Prefs} \; \mathit{or} \; (a<b) \in \mathit{Prefs})$} 
		\ForAll{$\mathit{Prefs}\in\mathit{PSet}$} 
		\State $\mathit{PSet'} \gets \mathit{PSet'} \cup \{\mathit{Prefs} \cup \{b < a\}\}$ 
		\State $\mathit{PSet'} \gets \mathit{PSet'} \cup \{\mathit{Prefs} \cup \{b = a\}\}$ 
		\State $\mathit{PSet'} \gets \mathit{PSet'} \cup \{\mathit{Prefs} \cup \{a < b\}\}$ 
		\EndFor
		\State $\mathit{PSet} \gets \mathit{PSet'}$, $\mathit{PSet'}\gets\emptyset$
		
		\EndIf
		\EndFor
			
		\EndIf
		\EndFor
		\EndFor
		
		\State \Return $\mathit{PSet}$
		\EndFunction
	\end{algorithmic}
	\caption{Compute preferences (Case 3)}
	\label{alg:computing-preferences-case3}	
\end{algorithm}

We establish that our approach is sound (i.e., all its outputs are correct) and complete (i.e., it outputs all possible solutions with the three cases).
\begin{theorem}
	Algorithm~\ref{alg:computing-preferences} is sound in that given an ABA framework and an extension $E$ (under a given semantic $\sigma$) as input, every output preference set $\mathit{Prefs}\in\mathit{PSet}$ when used with the ABA framework, i.e., \ABAplus framework, results in the input $E$ (under a given $<$-$\sigma$ semantic).
\end{theorem}
\begin{proof}
We prove this by exploring all cases and how these are handled by Algorithms $2$-$4$.
Each set of preferences computed for each subset of assumptions $\alpha, \beta, \gamma \subseteq \mathcal{A}$ is such that $\alpha, \gamma \subseteq E, \beta\cap E = \emptyset$. We proceed to show how each of the auxiliary Algorithms $2$-$4$ help us achieve this.

Algorithm~\ref{alg:computing-preferences-case1} computing each case $1$ preference set with at least one preference of the form $b<a, a\in A, A \subseteq \alpha, \forall b \in B, B \subseteq \beta, B \vdash_R \bar{a}$ ensures that the following holds:
\begin{inparaenum}[(1)]
	\item There is no unattacked $C \in \gamma$ such that $\exists c \in C, c\neq a$, $C \vdash_R \bar{b}$ where $b \in B$ (lines $6$-$7$).
	\item $a \in E$ since $a$ is preferred to at least one of its attacking assumption $b$ in $B$, which invalidates the attack $B \vdash_R \bar{a}$.
	\item Since the input extension $E$ consists of conflict free assumptions, if $a\in E$ then its attacking assumption $b\not\in E$ where $b \in B$. This supports that $\beta\cap E=\emptyset$.
\end{inparaenum}

Algorithm~\ref{alg:computing-preferences-case2} computing each case $2$ preferences of the form $b<a, b=a, 
\forall a\in A, A \subseteq \alpha, b \in B, B \subseteq \beta, A \vdash_R \bar{b}, B \not\vdash_R \bar{a}$ ensures the following holds:
\begin{inparaenum}[(1)]
	\item Since $A$ attacks $B$ and $B$ does not attack $A$, we have two different preferences between $a$ and $b$, namely, $b < a, b = a$. Therefore $a\in E$ with respect to each of these preferences.
	\item Preferences $b < a, b = a$ will be in different preference sets, as per lines 8 and 9. 
	We will have $\mathit{Prefs}_1\gets\mathit{Prefs}\cup \{b < a\}$ and
	$\mathit{Prefs}_2\gets\mathit{Prefs}\cup \{b = a\}$, where $\mathit{Prefs}$ consists of preferences of case $1$.
\end{inparaenum}

Algorithm~\ref{alg:computing-preferences-case3} computing each case $3$ preferences of the form $b<a, b=a, a<b, a \in A, A \subseteq \alpha, \forall b \in B, B \subseteq \beta, C \subseteq \gamma$ such that $C$ is unattacked, $B \vdash_R \bar{a}, C \vdash_R \bar{b}$ ensures the following holds:
\begin{inparaenum}[(1)]
	\item  Since $C$ defends $A$ from the attack of $B$, we have three different preferences between $a$ and $b$, namely, $b<a$, $b=a$ and $a<b$. Therefore $a \in E$ with respect to each of these preferences. 
	\item Preferences $b < a, b = a, a < b$ will be in different preference sets, as per lines $11$-$13$. 
	We will have $\mathit{Prefs}_1\gets\mathit{Prefs}\cup \{b < a\}$, $\mathit{Prefs}_2\gets\mathit{Prefs}\cup \{b = a\}$ and $\mathit{Prefs}_3\gets\mathit{Prefs}\cup \{a < b\}$, where $\mathit{Prefs}$ consists of preferences of cases $1$ and $2$.
\end{inparaenum}
\end{proof}

\begin{theorem}
	Algorithm~\ref{alg:computing-preferences} is complete in that given an ABA framework and an extension $E$ (under a given semantic $\sigma$) as input, if there is a preference set $\mathit{Prefs}\in\mathit{PSet}$ which when used with the ABA framework, i.e., \ABAplus framework, results in the input $E$ (under a given $<$-$\sigma$ semantic), then Algorithm~\ref{alg:computing-preferences} will find it. 
\end{theorem}

\begin{proof}
	
Similar to above, we prove this by exploring all cases and how these are handled by algorithms $2$-$4$. We find all sets of preferences computed for each subset of assumptions $\alpha, \beta, \gamma \subseteq \mathcal{A}$ such that $\alpha, \gamma \subseteq E, \beta\cap E = \emptyset$. We proceed to show how each of the auxiliary Algorithms $2$-$4$ help us achieve this.

Algorithm~\ref{alg:computing-preferences-case1} computes all case $1$ preference sets with at least one preference of the form $b<a, \: a\in A, \: A \subseteq \alpha, \: \forall b \in B, \: B \subseteq \beta, \: B \vdash_R \bar{a}$. Lines $3$-$24$ exhaustively search for $a \in E$ for which there is an attacker $B$ (not attacked by any unattacked $C$ such that $\nexists c \in C, \: c \neq a$). If there are such $a, b\in\mathcal{A}$ where $b \in B$, the algorithm will find them and add at least one preference of the form $b < a$ to a set of preferences.

Algorithm~\ref{alg:computing-preferences-case2} computes all case $2$ preferences of the form $b<a, \: b=a, \: \forall a\in A, \: A \subseteq \alpha, \: b \in B, \: B \subseteq \beta, \: A \vdash_R \bar{b}, \: B \not\vdash_R \bar{a}$.
Lines $3$-$10$ exhaustively search for $a \in E$ for which there is an attacked set of assumptions $B$ and $B$ does not attack $A$ where $a \in A$. If there are such $a, b\in\mathcal{A}$ where $b \in B$, the algorithm will find them and add each $b<a, \: b=a$ to a different set of preferences.

Algorithm~\ref{alg:computing-preferences-case3} computes all case $3$ preferences of the form $b<a, \: b=a, \: a<b, \: a \in A, \: A \subseteq \alpha, \: \forall b \in B, \: B \subseteq \beta, \: C \subseteq \gamma$ such that $C$ is unattacked, $B \vdash_R \bar{a}, \: C \vdash_R \bar{b}$. Lines $3$-$14$ exhaustively search for $a \in E$ for which there is an attacker $B$ and there is a defender $C$ that attacks $B$. If there are such $a, b\in\mathcal{A}$ where $b \in B$, the algorithm will find them and add each $b<a, b=a, a<b$ to a different set of preferences.
\end{proof}

\noindent Given an ABA framework and an extension $E$ (under a given semantic $\sigma$), the complexity (worst-case) of finding all possible sets of preferences between assumptions is exponential due to the exponential complexity of the constituent functions of Algorithm~\ref{alg:computing-preferences}.


\section{Evaluation and Results}
\label{sec:evaluation}
In this section, we present an illustrative example to demonstrate and evaluate Algorithm~\ref{alg:computing-preferences}. 
	Let there be an ABA framework $\mathcal{F}$ of Example~\ref{example:aba} and its corresponding assumption graph $\mathcal{G}_1$ shown in Figure~\ref{fig:aba_arg_graph}. 
$\mathcal{F}$ has two preferred extensions $E_1 =\{a,c\}$, $E_2 =\{b,c\}$.
To demonstrate Algorithm~\ref{alg:computing-preferences} we consider the preferred extension $E_1 =\{a,c\}$ for computing preferences~\footnote{Due to space restrictions we only demonstrate this on preferred extensions, but the approach works on all conflict-free extensions.}. Table~\ref{table:extension2-demonstration} shows the preferences computed as follows:
	\begin{inparaenum}[(i)]
		\item At line $2$, Algorithm~\ref{alg:computing-preferences-case1} is called, which returns the sets of case $1$ preferences. 
		\item At line $3$, Algorithm~\ref{alg:computing-preferences-case2} is called, which returns the sets of preferences with cases $1$ and $2$ combined together.
		\item Finally at line $4$, Algorithm~\ref{alg:computing-preferences-case3} is called, which returns the sets of preferences with cases $1$, $2$ and $3$ combined together.
	\end{inparaenum}

	\begin{table}[!h]
	\begin{minipage}[b]{0.4\linewidth}
	\centering\renewcommand\arraystretch{1.2}
		\caption{Computing Preferences for Extension $E_1 =\{a,c\}$}
		\label{table:extension2-demonstration}			
		\hspace*{-10em}
		\begin{tabular}{|c|p{25mm}|}
			\hline
			Line No.& Preference Sets\\\hline
			$2$ & $\{b<a, c=a\}$ \newline $\{b=a, c<a\}$ \newline $\{b<a, c<a\}$\\ \hline
			$3$ & $\{b<a, c=a, b<c\} \newline \{b<a, c=a, b=c\} \newline \{b=a, c<a, b<c\} \newline \{b=a, c<a, b=c\}
			\newline \{b<a, c<a, b<c\}
			\newline \{b<a, c<a, b=c\}$
			\\ \hline
			$4$ & 
			$\{b<a, c=a, b<c\} \newline \{b<a, c=a, b=c\} \newline \{b=a, c<a, b<c\} \newline \{b=a, c<a, b=c\}
			\newline \{b<a, c<a, b<c\}
		\newline \{b<a, c<a, b=c\}$
			\\
			\hline
		\end{tabular}
		\end{minipage}
	\begin{minipage}[b]{0.4\linewidth}		
	\centering\renewcommand\arraystretch{1.2}
		\caption{Preferences for the Preferred extensions $\{a,c\}$ and $\{b,c\}$}
	\label{table:extensions-preferences}
	\hspace*{-6em}
	\begin{tabular}{|p{12mm}|p{27mm}|p{10mm}|p{10mm}|}
		\hline
		Extensions & Preference Sets & Unique & Common  \\ \hline
		$\{a,c\}$ &  $\{b<a, c=a, b<c\} \newline \{b<a, c=a, b=c\} \newline \{b=a, c<a, b<c\} \newline \{b=a, c<a, b=c\}
		\newline \{b<a, c<a, b<c\}
		\newline \{b<a, c<a, b=c\}$
		& $b<a$\newline $c<a$ \newline $b<c$ & $b=a$ \newline $b=c$ \newline $c=a$
		\\
		\cline{1-3}
		$\{b,c\}$ & 	
        $\{a<b, c=b, a<c\} \newline \{a<b, c=b, a=c\} \newline \{a=b, c<b, a<c\} \newline \{a=b, c<b, a=c\} \newline \{a<b, c<b, a<c \} \newline \{a<b, c<b, a=c \}$		
		& $a<b$ \newline $c<b$ \newline $a<c$ &
		\\
		\hline
	\end{tabular}
	\end{minipage}
	\end{table}
	\raggedbottom
\begin{algorithm}[H]
	\caption{Algorithm for Computing Unique Preferences}
	\label{alg:computing-preferences-unique}
	\begin{algorithmic}[1]
		\Require $\mathit{PrefSet_1}$, set of preference sets for extension $E_1$.
		\Require $\mathit{PrefSets}$, set consisting of the set of preference sets for all other given extensions except $E_1$.
		\Ensure $\mathit{UniquePrefs}$, unique preferences for $E_1$.
		\Function{ComputeUniquePreferences}{$\mathit{PrefSet_1},\mathit{PrefSets}$}
		\ForAll{$\mathit{Prefs_1} \in \mathit{PrefSet_1}$} 
		\ForAll{$p \in \mathit{Prefs_1}$} 
		\If{$\forall \mathit{PrefSet} \in \mathit{PrefSets} \;$ 
		
		$\nexists \mathit{Prefs} \in \mathit{PrefSet} \; s.t. \; p \in \mathit{Prefs}$}
		\State $\mathit{UniquePrefs} \gets \mathit{UniquePrefs} \cup p$
		\EndIf
		\EndFor	
		\EndFor
		\Return $\mathit{UniquePrefs}$ 
		\EndFunction
	\end{algorithmic}
\end{algorithm}

Table~\ref{table:extensions-preferences} presents the sets of preferences for the two preferred extensions $\{a,c\}$ and $\{b,c\}$, and the unique and common preferences for an extension in comparison to another extension. The unique preferences for an extension in comparison to all other extensions for a given semantic can be computed by Algorithm~\ref{alg:computing-preferences-unique}. Furthermore, the common preferences for an extension in comparison to all other extensions can be computed by Algorithm~\ref{alg:computing-preferences-common}.

An important reason for finding all possible preferences is that, in a multi-extension semantic, if there are more than two extensions, e.g., $E_1, E_2, E_3$, then there might be a preference $\mathit{Pref_1}$ that is unique for $E_1$ in comparison to the preferences of $E_2$, and a preference $\mathit{Pref_2}$ that is unique for $E_1$ in comparison to the preferences of $E_3$, and the set of preferences $\{\mathit{Pref_1}, \mathit{Pref_2}\}$ would still result in the extension $E_1$ and not $E_2$ or $E_3$.
\begin{algorithm}[H]
	\caption{Algorithm for Computing Common Preferences}
	\label{alg:computing-preferences-common}	
	\begin{algorithmic}[1]
		\Require $\mathit{PrefSet_1}$, set of preference sets for extension $E_1$.
		\Require $\mathit{PrefSets}$, set consisting of the set of preference sets for all other given extensions except $E_1$.
		\Ensure $\mathit{CommonPrefs}$, common preferences for $E_1$.
		\Function{ComputeCommonPreferences}{$\mathit{PrefSet_1},\mathit{PrefSets}$}
		\ForAll{$\mathit{Prefs_1} \in \mathit{PrefSet_1}$} 
		\ForAll{$p \in \mathit{Prefs_1}$} 
		\If{$\forall \mathit{PrefSet} \in \mathit{PrefSets} \;$ 
		
		$\exists \mathit{Prefs} \in \mathit{PrefSet} \; s.t. \; p \in \mathit{Prefs}$}		
		\State $\mathit{CommonPrefs} \gets \mathit{CommonPrefs} \cup p$
		\EndIf
		\EndFor	
		\EndFor

		\Return $\mathit{CommonPrefs}$ 
		\EndFunction
	\end{algorithmic}
\end{algorithm}

\section{Related Work}
\label{sec:related-work}
~\cite{mahesar_prima18} propose an extension-based approach for computing preferences for a given set of conflict-free arguments in an abstract argumentation framework. However, the preferences are computed over abstract arguments which are atomic and the algorithms do no not take into account structural components of arguments such as assumptions, contraries and inference rules in the context of structured argumentation formalisms in particular ABA and ABA\textsuperscript{+}.

Value-based argumentation framework (VAF)~\cite{bench-capon03-persuasion,Bench-CaponDD07,AiriauBEMR17} extends a standard argumentation framework to take into account values and aspirations to allow divergent opinions for different audiences. Furthermore, value-based argumentation frameworks (VAFs) have been extended to take into account the possibility that arguments may support multiple values, and therefore, various types of preferences over values could be considered in order to deal with real world situations~\cite{KaciT08}.

\ABAplus~\cite{CyrasT16} generalises preference-based argumentation framework (PAF)~\cite{AMGOUD2014585} and improves assumption-based argumentation equipped with preferences (p\_ABA)~\cite{Wakaki14}. ASPIC\textsuperscript{+}~\cite{ModgilP14} encompasses many key elements of structured argumentation such as strict and defeasible rules, general contrariness mapping and various forms of attacks as well as preferences. Another variation is an extended argumentation framework (EAF)~\cite{MODGIL2009901} that considers the case where arguments can express preferences between other arguments. While the above structured argumentation frameworks allow handling of preferences over argument components, the main limitation is that preferences need to be stated in advance. 

\section{Conclusions and Future Work}
\label{sec:conclusion}
In this paper, we presented a novel solution for an important problem of preference elicitation in structured argumentation. 
The main contributions of our work are as follows:
\begin{inparaenum}[(i)]
    \item While previous work considered eliciting preferences in abstract argumentation, in this work, we focused on eliciting preferences in structured argumentation, i.e., ABA. Hence, the preferences are computed over assumptions rather than abstract arguments.
    \item We presented a novel algorithm for the structured setting in ABA. We established, its soundness, completeness and complexity (worst-case).
    \item We presented algorithms for finding unique and common preferences of an extension in comparison to the preferences of all other extensions for a given multi-extension semantic.
\end{inparaenum}

Currently, we provide a mechanism for filtering preferences via Algorithms~\ref{alg:computing-preferences-unique}-~\ref{alg:computing-preferences-common} by finding unique and common preferences, we aim to extend this in the future to identify useful patterns or combinations of preferences such as the ones specified in Section~\ref{sec:evaluation}. 
Moreover, it will be interesting to find out what case of preferences are more likely to be unique or common for a multi-extension semantic. 

Furthermore, we would like to extend the approach presented in this paper to other structured argumentation frameworks~\cite{Besnard2014-BESITS}, in particular ASPIC\textsuperscript{+}~\cite{ModgilP14}.
We intend to investigate the relationship between extension enforcement~\cite{ecai-Baumann12,CMKMM2015} and our work. As one application, our work could be used in dialogue strategies~\cite{ki-Thimm14,ijcai-RienstraTO13} where an agent can have the capability of inferring preferences and reach her goal if she enforces at least one of several desired sets of arguments with the application of preferences. 


 \bibliographystyle{abbrv}
\bibliography{comma20}
\end{document}